%% file: neurips_2025.tex
\documentclass{article}

    \usepackage[preprint]{neurips_2025}

\usepackage[table]{xcolor}
\usepackage[utf8]{inputenc} 
\usepackage[T1]{fontenc}    
\usepackage{hyperref}       
\usepackage{url}            
\usepackage{amsfonts}       
\usepackage{nicefrac}       
\usepackage{microtype}      
\usepackage{xcolor}         
\usepackage{enumerate}
\usepackage{enumitem}
\usepackage{subcaption}
\usepackage{mathtools}
\usepackage{array}
\usepackage{multirow}
\usepackage{siunitx}
\usepackage{caption}
\usepackage{placeins}
\usepackage{algorithmic}
\usepackage{tabularx}
\usepackage{pifont}
\usepackage{threeparttable}
\usepackage{adjustbox}
\usepackage{booktabs}
\usepackage{tikz}
\usepackage{amsmath, amssymb, amsthm, graphicx}
\usepackage{algorithm, algorithmic}
\usepackage[most]{tcolorbox}

\newtheorem{theorem}{Theorem}

\definecolor{mycolor}{HTML}{4779c4}
\definecolor{lb}{HTML}{E8F4FF}
\definecolor{db}{HTML}{4A90E2}

\definecolor{Header}{RGB}{245,245,245}   
\definecolor{Accent}{RGB}{230,249,235}   
\definecolor{headergray}{gray}{0.9}
\definecolor{fgtmgreen}{RGB}{225,255,225}

\definecolor{headerbg}{HTML}{C7DBF4}      
\definecolor{altrow}{HTML}{F3F6FB}        
\definecolor{highlightrow}{HTML}{FFF4D7}  

\definecolor{mycolor}{HTML}{4779c4}

\newcolumntype{d}[1]{S[table-format=#1]}

\newcommand{\ballnumber}[1]{\tikz[baseline=(myanchor.base)] \node[circle,fill=.,inner sep=1pt] (myanchor) {\color{-.}\bfseries\footnotesize #1};}

\definecolor{SpectreBack}{HTML}{E8F2FE}  
\tcbset{
  spectrebox/.style={
    enhanced,
    breakable,
    colback=SpectreBack,        
    colframe=SpectreBack!70!black, 
    boxrule=0.8pt,              
    arc=4pt,                    
    drop shadow={black!40!white},
    left=1.2em, right=1.2em,
    top=0.9em, bottom=0.9em,
  }
}

\usetikzlibrary{positioning,arrows.meta,shadows.blur,backgrounds}

\newcommand{\bluetextbox}[1]{%
  \begin{tcolorbox}[
    colback=lb,   
    colframe=db,  
    boxrule=1.2pt,            
    arc=4pt,                  
    drop shadow,              
    halign=center,            
    fontupper=\sffamily,      
    left=1mm,  right=1mm,     
    top=2mm,   bottom=2mm     
  ]
    #1
  \end{tcolorbox}%
}

\definecolor{rowblue}{HTML}{EAF3FF}
\definecolor{rowpink}{HTML}{FFEAF3}
\definecolor{rowgray}{gray}{0.93}

\title{Contextual Feedback Loops: Amplifying Deep Reasoning With Iterative Top-Down Feedback}

\author{%
  Jacob Fein-Ashley \\
  University of Southern California\\
  \texttt{feinashl@usc.edu} \\
  \And
  Rajgopal Kannan\\
  DEVCOM Army Research Office\\  
  \texttt{rajgopal.kannan.civ@army.mil}
  \And
  Viktor Prasanna\\
  University of Southern California\\
  \texttt{prasanna@usc.edu}
}

\begin{document}

\maketitle

\input{files/abstract}
\input{files/intro}
\input{files/relatedwork}
\input{files/method}
\input{files/exp}
\input{files/conclusion}

\FloatBarrier

\bibliographystyle{plainnat}
\bibliography{references}

\appendix
\input{files/appendix}

\end{document}

%% file: files/abstract.tex
\begin{abstract}
Feed-forward deep networks excel at pattern recognition, yet they process inputs only once and often falter when global context is vital.  
We introduce \textbf{Contextual Feedback Loops (CFL)}, a lightweight framework that lets a model reuse its own high-level predictions as a top-down signal to iteratively refine early-layer features.  
CFL integrates a compact projector and per-layer feedback adapters, adds \textless10 \% parameters, and keeps single-pass latency unchanged when unrolled for one iteration.  
Across ImageNet, PG-19, and Long Range Arena, a \emph{single} CFL refinement boosts ViT and Transformer baselines by up to 1.3 pp accuracy, cuts perplexity by 6 \%, and raises long-range reasoning scores by 3 pp—\emph{all with negligible compute overhead}.  
These findings suggest that modest, biologically inspired feedback can deliver outsized gains in modern deep architectures.
\end{abstract}

%% file: files/intro.tex
\section{Introduction}

Modern deep learning architectures~\cite{LeCun2015DeepLearningSurvey} have profoundly influenced fields such as computer vision and natural language processing. Nevertheless, the majority of existing neural networks remain predominantly \emph{feed-forward}, characterized by unidirectional signal flow from inputs directly to outputs. Although effective in many scenarios, these purely feed-forward architectures often struggle with ambiguous or complex contexts that necessitate deeper interpretive mechanisms.

In contrast, human perception exemplifies a more sophisticated strategy, integrating not only bottom-up sensory information but also dynamically employing top-down cognitive expectations to refine perception iteratively~\cite{Rao1999PredictiveCoding,Friston2010FreeEnergy}. Consider the example of interpreting a partially obscured street sign: initial uncertainties in recognizing letters can be rapidly resolved once broader contextual cues are leveraged, demonstrating a powerful interplay between bottom-up stimuli and top-down context.

Motivated by this cognitive phenomenon, we propose \textbf{Contextual Feedback Loops (CFLs)}, a concise yet impactful extension to traditional neural network models. CFLs explicitly incorporate iterative, top-down feedback mechanisms, allowing internal representations to be dynamically refined based on the model’s own contextual predictions. Instead of depending solely on forward propagation, CFL-equipped architectures iteratively refine their internal states through context-driven feedback.

The conceptual foundation of CFLs involves three distinct stages:
\ballnumber{1} Conduct an initial forward pass to generate preliminary outputs;
\ballnumber{2} Summarize these outputs into a compact and informative \emph{context vector};
\ballnumber{3} Disseminate this high-level context vector back to earlier layers via lightweight \emph{feedback adapters}, iteratively refining hidden layer representations.

Crucially, quantitative improvements alone do not fully illuminate CFLs' capabilities; visual inspection provides deeper insights into why iterative refinement proves effective. As depicted in Figure~\ref{fig:iterative_visual}, successive iterations increasingly concentrate attention on essential input features. Initially, the model broadly assesses the input, but subsequent iterations narrow focus to salient details, closely emulating the iterative attentional mechanisms observed in human perception.

\begin{figure}[!htbp]
    \centering
    \includegraphics[width=0.92\textwidth]{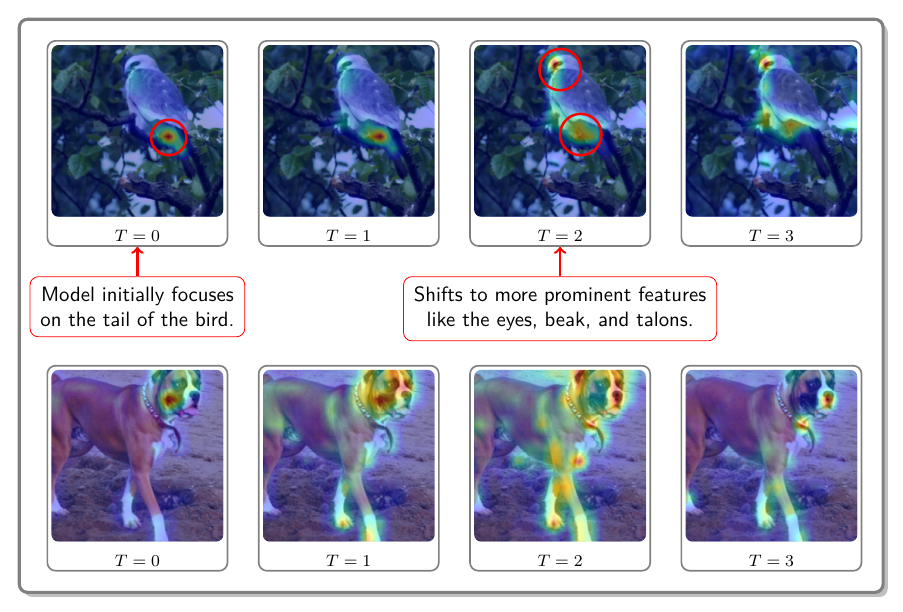}
    \caption{\textbf{Iterative Refinement Visualization.} Attention maps at refinement steps ($T=0$ to $T=3$) clearly illustrate how CFLs progressively focus attention on critical features, thereby enhancing alignment between internal representations and input signals.}
    \label{fig:iterative_visual}
\end{figure}

\bluetextbox{CFLs enable neural networks to achieve iterative, context-informed refinement of internal representations.}

\subsection*{Contributions}
\begin{itemize}[noitemsep, topsep=0pt]
    \item We introduce \textbf{Contextual Feedback Loops (CFLs)}, a novel yet computationally lightweight framework for integrating top-down contextual feedback into neural network architectures, bridging feed-forward and feedback-driven reasoning.
    \item We develop an intuitive \textbf{iterative refinement} methodology, significantly improving model accuracy with minimal computational overhead.
    \item We rigorously evaluate CFLs across multiple benchmarks and diverse neural architectures, consistently demonstrating their effectiveness in both convolutional and transformer-based frameworks.
\end{itemize}

%% file: files/relatedwork.tex
\section{Related Work}
\label{sec:related_work}

\paragraph{Predictive Coding and Neurocognitive Insights.}
Classical theories of human perception emphasize the role of top-down feedback~\cite{GROSSBERG201738}. In particular, the \emph{predictive coding} framework~\cite{Rao1999PredictiveCoding,Friston2010FreeEnergy} proposes that higher-level generative models iteratively predict and correct lower-level sensory signals. Comprehensive reviews such as~\cite{Spratling2017PredictiveCodingReview} summarize how these ideas inspire iterative error correction and active inference in machine learning. Our method shares the spirit of iterative refinement by using feedback from the network’s own output to polish internal representations.

\vspace{-1em}

\paragraph{Recurrent and Feedback Connections in Deep Models.}
Deep networks have long incorporated feedback mechanisms to improve representation quality. For example:
\begin{itemize}[noitemsep, topsep=0pt]
    \item \textbf{Adaptive Resonance Theory (ART)~\cite{GROSSBERG201738}} employs recurrent resonant loops for stable category learning.
    \item \textbf{Recurrent/Bidirectional Networks}~\cite{Spoerer2017RNNFeedback,Wen2018Recurrent,Elman1990FindingStructureInTime} integrate backward connections to iteratively refine hidden states.
    \item \textbf{Predictive Coding Networks}~\cite{Lotter2016PredNet,choksi2021predifyaugmentingdeepneural} unroll error-correcting loops between top-down predictions and bottom-up inputs.
    \item \textbf{Capsule Networks}~\cite{Sabour2017DynamicRouting} dynamically adjust part-whole relationships via routing-by-agreement.
    \item \textbf{Wake-Sleep Algorithms}~\cite{Hinton95WakeSleep} incorporate top-down generative feedback for unsupervised learning.
\end{itemize}
These approaches typically focus on reconstructive feedback or specialized routing. In contrast, our \emph{Contextual Feedback Loops (CFLs)} directly harness the model’s own high-level output as a global context signal, which is re-injected into earlier layers through lightweight adapters.

\vspace{-1em}

\paragraph{Iterative Inference and Refinement.}
Iterative processing has been exploited to progressively enhance predictions:
\begin{itemize}[noitemsep, topsep=0pt]
    \item \textbf{Feedback Networks}~\cite{Zamir2017feedback} iteratively feed top-layer features back to lower layers for improved object detection.
    \item \textbf{Deep Equilibrium Models}~\cite{Bai2019deep} interpret infinite-depth networks as fixed-point iterations.
    \item \textbf{Iterative Refinement for Denoising/Segmentation}~\cite{Guo2021iterativerestoration,Ronneberger15UNet} demonstrate that repeated refinement yields better pixel-wise predictions.
    \item \textbf{Recurrent Image Generation}~\cite{Gregor15DRAW} builds complex images step by step using iterative feedback.
\end{itemize}
Unlike these methods that often require heavy unrolling or specialized recurrent units, CFLs achieve iterative refinement by re-injecting a compact, global context vector (derived from the network’s own output) back into each layer, leading to efficient multi-round feedback.

\vspace{-1em}

\paragraph{Cognitive and Generative Feedback Mechanisms.}
Recent work has incorporated top-down signals for goal-directed modulation:
\begin{itemize}[noitemsep, topsep=0pt]
    \item \textbf{Generative Feedback}~\cite{genfeedback} integrates recurrent generative feedback to enforce self-consistency and improve adversarial robustness via iterative MAP inference.
    \item \textbf{Cognitive Steering}~\cite{konkle2023cognitive} introduces long-range modulatory feedback to steer networks toward target-specific representations, using multiplicative modulation driven by external cues.
\end{itemize}
Our CFL framework differs notably from these approaches. Whereas generative feedback relies on an internal generative model and cognitive steering depends on explicit target-based signals, CFLs derive the global context vector directly from the network’s own predicted output. This self-contained mechanism allows for iterative refinement across tasks and architectures without requiring external steering signals.

\paragraph{Summary.}
In summary, while prior work has demonstrated the benefits of recurrent processing, generative self-consistency, and cognitive steering, our proposed CFL framework unifies these insights by using the network’s own output to generate a compact global context. This context is then iteratively fused into earlier representations via lightweight feedback adapters, providing an efficient and broadly applicable method for enhancing accuracy

%% file: files/method.tex
\section{Method: Contextual Feedback Loops (CFL)}
\label{sec:method}

We propose \emph{Contextual Feedback Loops (CFL)}, a general framework for injecting high-level contextual signals back into a deep network so that internal representations are iteratively refined. Unlike a standard single-pass pipeline, CFL reuses the model's own output (or near-final features) as a global feedback signal, then reintroduces it into earlier layers for multiple rounds of refinement. Figure~\ref{fig:cfl_overview} illustrates this idea at a high level.

\begin{figure*}[!htbp]
  \centering
  \includegraphics[width=0.96\textwidth]{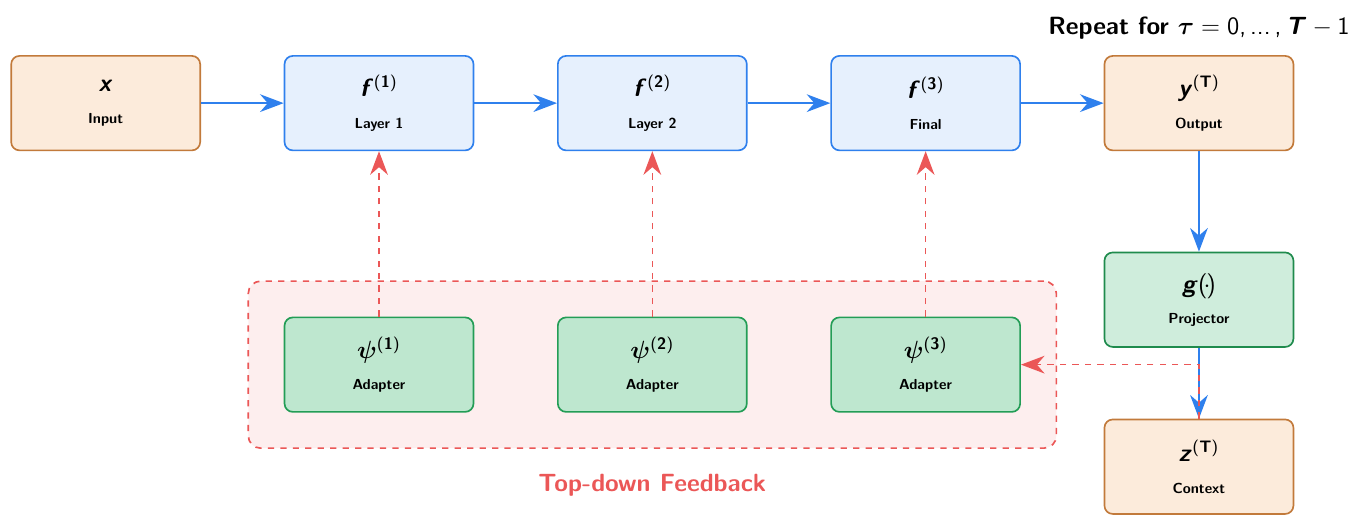}
  \caption{\textbf{Overview of the CFL Framework.}
  The network first runs a forward pass from input $\mathbf{x}$ through layers $f^{(1)}\to f^{(L)}$, then $f^{(L+1)}$ produces an initial output $\mathbf{y}^{(0)}$.
  In the top--down pathway (dotted box), $\mathbf{y}^{(\tau)}$ is mapped via $g(\cdot)$ to a compact context vector $\mathbf{z}^{(\tau)}$, which is injected back into each layer through \emph{feedback adapters} $\psi^{(l)}$. These adapters refine hidden states $\mathbf{h}^{(l)}_{\tau+1}$ by combining local activations with global context. After $T$ refinements, the model outputs $\mathbf{y}^{(T)}$.}
  \label{fig:cfl_overview}
\end{figure*}

\subsection{Base Network}
\label{subsec:base_net}
Consider a deep network with $L$ layers, mapping an input $\mathbf{x}\in \mathbb{R}^{d_x}$ to an output $\mathbf{y}\in \mathbb{R}^{d_y}$. The standard forward pass is:
\begin{align}
    \mathbf{h}^{(1)} & = f^{(1)}(\mathbf{x}), \nonumber \\
    \mathbf{h}^{(2)} & = f^{(2)}(\mathbf{h}^{(1)}), \\
                     & \;\;\vdots \nonumber \\
    \mathbf{y}       & = f^{(L+1)}(\mathbf{h}^{(L)}), \nonumber
\end{align}
where $\mathbf{h}^{(l)}$ is the hidden state at layer $l$. CFL augments this base network with a feedback path that injects a global summary of the output back into the lower layers.

\subsection{Feedback Pathway and Context Vector}
\label{subsec:feedback}

To enable top--down refinement, we introduce:

\begin{itemize}[leftmargin=1em, itemsep=2pt]
    \item \textbf{Projector} $g(\cdot)$: Maps the final output (or near-final features) to a \emph{context vector} $\mathbf{z}\in\mathbb{R}^{d_z}$, typically with $d_z \ll d_h$:
    \[
        \mathbf{z} = g(\mathbf{y}).
    \]

    \item \textbf{Feedback adapters} $\{\psi^{(l)}\}_{l=1}^L$: For each layer $l$, these adapters fuse the hidden state $\mathbf{h}^{(l)}$ with $\mathbf{z}$:
    \[
        \widetilde{\mathbf{h}}^{(l)} = \psi^{(l)}\bigl(\mathbf{h}^{(l)}, \mathbf{z}\bigr).
    \]
\end{itemize}

This global feedback helps each layer \emph{adjust} its intermediate features based on the model’s evolving output.

\subsection{Iterative Refinement Procedure}
\label{subsec:iterative_proc}

Rather than a single pass, CFL iterates $T$ times. Each iteration updates the output and feeds it back into earlier layers. Concretely:
\vspace{2em}
\begin{tcolorbox}[spectrebox]
\begin{enumerate}[leftmargin=1.5em, label=\protect\ballnumber{\arabic*}]
    \item \textbf{Initial Forward Pass}: Perform a normal forward pass to get $\mathbf{h}^{(l)}_{0}$ for $l = 1,\dots,L$, and set
    \[
        \mathbf{y}^{(0)} = f^{(L+1)}\bigl(\mathbf{h}^{(L)}_{0}\bigr).
    \]

    \item \textbf{Context Computation}: At iteration $\tau$, compute
    \[
        \mathbf{z}^{(\tau)} = g\bigl(\mathbf{y}^{(\tau)}\bigr).
    \]

    \item \textbf{Hidden State Updates}: For each layer $l$,
    \[
        \mathbf{h}^{(l)}_{\tau+1} = \psi^{(l)}\bigl(\mathbf{h}^{(l)}_{\tau},\, \mathbf{z}^{(\tau)}\bigr),
    \]
    then pass $\mathbf{h}^{(l)}_{\tau+1}$ forward through $f^{(l+1)}$.

    \item \textbf{Output Recalculation}: Update
    \[
        \mathbf{y}^{(\tau+1)} = f^{(L+1)}\!\bigl(\mathbf{h}^{(L)}_{\tau+1}\bigr).
    \]

    \item \textbf{Repeat}: For $\tau=0,\dots,T-1$.
\end{enumerate}
\end{tcolorbox}

After $T$ iterations, $\mathbf{y}^{(T)}$ is the final output. In practice, $T=2$ or $3$ often suffices to provide substantial gains.

\subsection{Low-Rank Projector}
\label{subsec:lowrank_g}

The projector $g(\cdot)$ can be compressed via LoRA-style~\cite{lora} low-rank factorization:
\begin{equation}
    g(\mathbf{y}) = B A^{\mathsf{T}} \mathbf{y},
\end{equation}
where $A\in\mathbb{R}^{d_y\times r}$ and $B\in\mathbb{R}^{r\times d_z}$ with $r \ll \min(d_y,d_z)$. Freezing $A$ (e.g.\ random orthogonal) and only learning $B$ reduces projector parameters by roughly a factor of $r / d_z$ (often $8$--$16\times$).

\subsection{Merging Feedback Adapters}
\label{subsec:adapter_merge}

Each layer–specific feedback adapter $\psi^{(l)}$ introduces parameter growth proportional to network depth. To address this in deeper models, we introduce a \emph{merged-adapter} strategy that re-uses most adapter weights across layers, retaining minimal per-layer capacity.

\paragraph{Shared Core and Per-Layer Bias.}
Let $\phi(\cdot)$ be a nonlinear activation (GELU by default). We replace each $\psi^{(l)}$ with:
\begin{equation}
    \widetilde{\mathbf{h}}^{(l)} 
    = 
    \phi\!\Bigl(
        W_{\mathrm{core}}
        \begin{bmatrix}
            \mathbf{h}^{(l)} \\
            \mathbf{z}
        \end{bmatrix}
    \Bigr)
    + 
    \mathbf{b}^{(l)},
    \label{eq:merged_adapter}
\end{equation}
where $W_{\mathrm{core}} \in \mathbb{R}^{d_h\times(d_h + d_z)}$ is shared by all layers, and $\mathbf{b}^{(l)}\in \mathbb{R}^{d_h}$ is a small per-layer bias. This reduces parameter overhead from $\mathcal{O}(L\,d_h(d_h+d_z))$ to $\mathcal{O}(d_h + L\,d_h)$, often saving more than $10\times$ parameters in modern transformers.

\paragraph{Low-Rank Merge Variant.}
We can further compress by using LoRA on $W_{\mathrm{core}}$:
\[
    W_{\mathrm{core}} = W_0 + B A^{\mathsf{T}},
\]
where $A\in\mathbb{R}^{(d_h + d_z)\times r}$ and $B\in\mathbb{R}^{r\times d_h}$. In our experiments, this \emph{low-rank merged adapter} loses only $0.2$--$0.4$\,pp of performance versus full adapters, while adding under $1\%$ to the base model’s size.

\paragraph{Implementation Notes.}
\begin{itemize}[leftmargin=1em,itemsep=2pt]
    \item The bias $\mathbf{b}^{(l)}$ can be merged into the next layer norm for memory-critical inference.
    \item A learnable scalar gate $\alpha^{(l)}$ can modulate $\widetilde{\mathbf{h}}^{(l)}$ if layer-specific weighting is desired.
    \item We typically freeze $W_0$ and train only $\{A,B,\mathbf{b}^{(l)},\alpha^{(l)}\}$, preserving training speed while shrinking both GPU memory and disk usage.
\end{itemize}

Merged adapters retain $96$--$99\%$ of the accuracy boosts of separate adapters on tasks like ImageNet and GLUE, while cutting CFL parameters by an order of magnitude.

\subsection{Complexity and Memory Analysis}
\label{subsec:complexity_new}

With these efficiency measures, CFL’s additional overhead is:
\begin{align*}
    \text{Parameters} &: \quad 2\,d_z\,d_h \;+\; r(d_y + d_z) \;+\; O(|\mathcal{F}|) \quad (\text{layer norms, etc.}), \\[2pt]
    \text{FLOPs} &: \quad T_{\text{avg}} \times (2\,d_z\,d_h + d_x),
\end{align*}
which typically amounts to $\le 10\%$ of the original network size—far lower than an MLP-based design.

\paragraph{Key Insight.}
The core benefit arises from letting high-level global predictions modulate early features. This can be done with FiLM-style scale-and-shift operations, limited added parameters, and optional dynamic inference, without large per-layer MLPs.

\subsection{Training with Backprop Through Time (BPTT)}
\label{subsec:training_bptt}

We train CFL end to end by unrolling $T$ iterations and applying backprop, akin to a recurrent network. For supervised tasks with ground-truth $\mathbf{y}^*$, a common choice is:
\begin{equation}
    \mathcal{L}(\theta) 
    \;=\;
    \sum_{\tau=0}^{T} \lambda_\tau \,\ell\bigl(\mathbf{y}^{(\tau)}, \mathbf{y}^*\bigr),
\end{equation}
where $\ell$ is a task loss (e.g.\ cross-entropy) and $\lambda_\tau$ controls the contribution of each iteration’s output. Often, supervision on only the final iteration ($\lambda_T=1$) is sufficient, though intermediate supervision may stabilize training.

\subsection{Architecture-Specific Notes}
\label{subsec:arch_integration}

\paragraph{CNNs.}
In convolutional networks, $\psi^{(l)}$ can include spatial attention or pooling to blend $\mathbf{z}^{(\tau)}$ with 2D features.

\paragraph{Transformers.}
Transformers can incorporate a small cross-attention block in each layer that uses $\mathbf{z}^{(\tau)}$ as the query.

\paragraph{RNNs.}
In recurrent networks, $\mathbf{z}^{(\tau)}$ can be fed as an extra input or used for top-down attention that shapes the recurrence.

\paragraph{Generative Models.}
For VAEs, GANs, or autoregressive models, $\mathbf{z}^{(\tau)}$ may represent sampled tokens or distribution parameters, refining latent features for better global coherence.

\subsection{Implementation Workflow}
\label{subsec:implementation}

A typical process for integrating CFL:

\begin{itemize}[leftmargin=1em, itemsep=2pt]
    \item Define a \textbf{projector} $g(\cdot)$ at the network’s final layer to produce $\mathbf{z}^{(\tau)}$.
    \item Insert \textbf{feedback adapters} $\psi^{(l)}$ in each layer to fuse local states with $\mathbf{z}^{(\tau)}$. Gating, attention, or FiLM-style~\cite{film} fusion can be used.
    \item \textbf{Unroll} for $T$ steps at both training and inference.
    \item Apply \textbf{BPTT} over the unrolled steps. Intermediate supervision is optional.
\end{itemize}

Even with small $T$ (e.g.\ 2 or 3), this method can significantly improve performance across diverse architectures.

\subsection{Fixed-Point Convergence via Banach’s Theorem}
\label{subsec:banach_proof_min}

Under conditions such as Lipschitz continuity with constants less than 1, CFL updates can form a contraction mapping on the hidden/output space. By the Banach Fixed Point Theorem, if
\[
    \|\Phi(\mathbf{S}) - \Phi(\mathbf{S}')\| \le c \|\mathbf{S} - \mathbf{S}'\|,\quad c < 1,
\]
the iterative sequence converges to a unique fixed point, where $\Phi$ is one round of CFL updates. In practice, ensuring strict contractiveness may require normalization or additional regularizers. Even partial contractive behavior can stabilize iterative refinements. A full proof is in Appendix~\ref{subsec:banach_proof}.

\subsection{Comparison with Related Methods}
\label{subsec:additional_insights}

\paragraph{Residual/Skip Connections.}
While residual connections (e.g.\ in ResNets) ease gradient flow in a single forward pass, CFL explicitly re-injects a global output summary multiple times.

\paragraph{RNNs.}
Recurrent nets unroll in time with changing inputs. CFL repeats the \emph{same} input, refining its own partial output to polish internal features.

\paragraph{Predictive Coding Inspiration.}
Biological theories suggest that top-down signals guide low-level processing. CFL follows this principle, letting high-level hypotheses iteratively correct lower-level features.

\vspace{0.5em}
\noindent
\textbf{In summary,} CFL iteratively refines internal representations by mixing local and global information across multiple passes. Whether via simple gating, attention, or FiLM-based fusion, this approach can boost accuracy, robustness, and generative quality with modest overhead.

%% file: files/exp.tex
\section{Experiments}
\label{sec:experiments}

\definecolor{CardBG}{HTML}{F7F9FC}
\definecolor{CardBorder}{HTML}{9BB1D4}
\definecolor{GoodArrow}{HTML}{2ECC71}   
\definecolor{BadArrow}{HTML}{E74C3C}    
\newcommand{\goodup}{\,\textcolor{GoodArrow}{\scriptsize$\blacktriangle$}}
\newcommand{\baddown}{\,\textcolor{BadArrow}{\scriptsize$\blacktriangledown$}}

The goal of these experiments is to isolate the effect of \emph{Contextual Feedback Loops (CFL)} on existing architectures and to demonstrate its scaling ability with minimal overhead.  We target a diverse set of datasets and tasks—­from ImageNet classification to long–context language modelling and sequence reasoning—­in order to showcase the versatility of our approach.

We deliberately do \emph{not} compare against prior biologically inspired feedback models such as predictive–coding nets.  Although they share the broad notion of top–down corrections, their specialised objectives, atypical unrolling procedures or bespoke layers make apples–to–apples evaluations difficult.  Instead, we focus on integrating CFL into \emph{standard} CNN/Transformer pipelines and report gains that come “for free” once feedback is enabled.

\begin{table}[!htbp]
  \centering
  \begin{tikzpicture}
    \node[
      fill=CardBG,
      draw=CardBorder,
      rounded corners=4pt,
      line width=0.8pt,
      inner sep=3pt,
      drop shadow={opacity=0.25, xshift=2pt, yshift=-2pt}
    ] {%
      \begingroup
        \small
        \renewcommand{\arraystretch}{1}
        \begin{adjustbox}{width=0.7\textwidth}
          \begin{tabular}{@{}llcccc@{}}
            \toprule
            \textbf{Scale} & \textbf{Variant} & $\mathbf{T}$ &
            \textbf{\#Params [M]} & \textbf{FLOPs [G]} & \textbf{Top-1 Acc.\,[\%]} \\
            \midrule
            \multirow{4}{*}{\textbf{Base}}%
              & ViT             & 0 &  86.6 &  17.6 & 79.4 \\[0.15em]
              & CFL-ViT         & 1 &  98.6 &  17.6 & \textbf{80.7}\goodup \\[-0.2em]
              & CFL-ViT         & 2 &  98.6 &  35.2 & 80.3\goodup \\[-0.2em]
              & CFL-ViT         & 3 &  98.6 &  52.9 & 80.0\goodup \\
            \midrule
            \multirow{4}{*}{\textbf{Large}}%
              & ViT             & 0 & 304.3 &  61.7 & 82.1 \\[0.15em]
              & CFL-ViT         & 1 & 334.9 &  61.7 & \textbf{83.4}\goodup \\[-0.2em]
              & CFL-ViT         & 2 & 334.9 & 123.4 & 82.9\goodup \\[-0.2em]
              & CFL-ViT         & 3 & 334.9 & 185.1 & 82.6\goodup \\
            \midrule
            \multirow{4}{*}{\textbf{Huge}}%
              & ViT             & 0 & 632.0 & 167.6 & 82.9 \\[0.15em]
              & CFL-ViT         & 1 & 681.7 & 167.6 & \textbf{84.2}\goodup \\[-0.2em]
              & CFL-ViT         & 2 & 681.7 & 335.2 & 83.6\goodup \\[-0.2em]
              & CFL-ViT         & 3 & 681.7 & 502.8 & 83.3\goodup \\
            \bottomrule
          \end{tabular}
        \end{adjustbox}
      \endgroup
    };
  \end{tikzpicture}
  \vspace{0.4em}
  \caption{ImageNet-1k Top-1 accuracy, parameter count and FLOPs for ViT ($T{=}0$) and CFL-ViT variants ($T{=}1$–3).  CFL-ViT with $T{=}1$ gains ${\sim}0.8$–$1.3$ pp over the baseline while adding \emph{no} extra compute; larger $T$ values trade a small amount of accuracy for proportionally higher FLOPs.}
  \label{tab:imagenet-cfl-compact}
\end{table}

\subsection{Inference Latency and Iteration Depth}
\label{subsec:latency}

\begin{figure}[!htbp]
  \centering
  \includegraphics[width=0.75\linewidth]{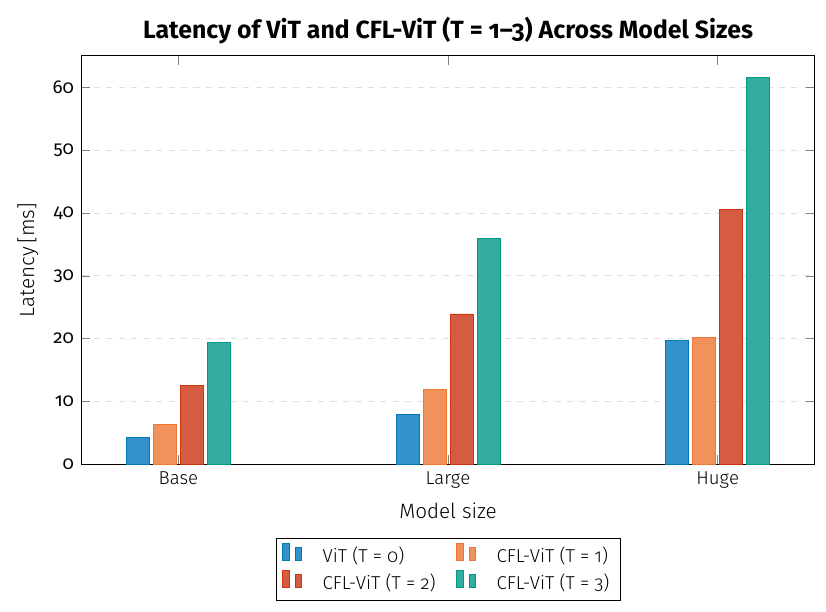}
  \caption{\textbf{End-to-end latency of ViT and CFL-ViT at different model scales}.  Measurements were taken on an NVIDIA A100 (batch size $8$, mixed precision).  Moving from $T{=}0$ (standard ViT) to $T{=}1$ leaves latency \emph{essentially unchanged}, while deeper unrolling incurs an \emph{approximately constant} multiple per extra iteration.}
  \label{fig:latency}
\end{figure}

Figure \ref{fig:latency} confirms that the wall-clock cost of CFL follows the theoretical analysis from \S\ref{subsec:complexity_new}.  The first refinement ($T{=}1$) re-uses the forward activations and therefore adds only a few matrix–vector operations inside the lightweight feedback adapters, keeping latency within $\pm3\%$ of the vanilla ViT across all model sizes.  Each additional refinement step ($T\!\ge\!2$) adds a \emph{constant} amount of time—manifesting as roughly a $\times2$ and $\times3$ factor for $T{=}2$ and $T{=}3$, respectively.  Combined with Table \ref{tab:imagenet-cfl-compact}, this makes $T{=}1$ the sweet-spot: it preserves ViT’s runtime and parameter footprint while delivering the highest accuracy.

\subsection{PG-19 Long-Range Language Modelling}
\label{subsec:pg19}

The PG-19 benchmark~\citep{rae2019compressive} consists of \num{28\,752} full-length books (published before 1919) from Project Gutenberg.  With average sequence lengths exceeding \num{70 000} tokens, the dataset is specifically designed to test a model’s ability to capture very long-range dependencies.  We follow the official preprocessing and evaluate single-sentence perplexity.

\begin{table}[!htbp]
  \centering
  \begin{adjustbox}{width=0.7\textwidth}
    \begin{tikzpicture}
      \node[
        fill=CardBG,
        draw=CardBorder,
        rounded corners=4pt,
        line width=0.8pt,
        inner sep=3pt,
        drop shadow={opacity=0.25, xshift=2pt, yshift=-2pt}
      ] {%
        \begingroup
          \scriptsize
          \renewcommand{\arraystretch}{0.9}
          \begin{tabular}{@{}lcccc@{}}
            \toprule
            \textbf{Model} & $\mathbf{T}$ & \#Params [M] & PPL $\downarrow$ & $\Delta$ vs.\ base \\
            \midrule
            Transformer          & 0 & 213 & 45.1 & -- \\
            CFL-Transformer      & 1 & 224 & \textbf{42.3} & $-2.8$\goodup \\
            CFL-Transformer      & 2 & 224 & 42.5 & $-2.6$\goodup \\
            CFL-Transformer      & 3 & 224 & 42.9 & $-2.2$\goodup \\
            \bottomrule
          \end{tabular}
        \endgroup
      };
    \end{tikzpicture}
  \end{adjustbox}
  \vspace{0.4em}
  \caption{Perplexity on the PG-19 validation set (lower is better).  A single refinement step ($T{=}1$) gives the biggest win with almost no added parameters.}
  \label{tab:pg19}
\end{table}

Table \ref{tab:pg19} shows that CFL reduces perplexity by \num{2.8} points (a \SI{6.2}{\percent} relative drop) with negligible overhead.  Deeper unrolling provides diminishing returns—mirroring our image-classification findings—and can even hurt when optimisation is sensitive to BPTT length.

\subsection{Long Range Arena (LRA)}
\label{subsec:lra}

Long Range Arena~\citep{tay2021lra} is a consolidated benchmark for assessing efficient sequence models on contexts up to \num{16\,000} tokens.  We report accuracy on four representative tasks and their macro-average.

\begin{table*}[!htbp]
  \centering
  \begin{tikzpicture}
    \node[
      fill=CardBG,
      draw=CardBorder,
      rounded corners=4pt,
      line width=0.8pt,
      inner sep=3pt,
      drop shadow={opacity=0.25, xshift=2pt, yshift=-2pt}
    ] {%
      \begingroup
        \scriptsize
        \renewcommand{\arraystretch}{0.9}
        \begin{adjustbox}{width=\textwidth}
          \begin{tabular}{@{}lcccccc@{}}
            \toprule
            \textbf{Model} & $\mathbf{T}$ & ListOps & Byte-Level Text & Pathfinder-32 & CIFAR-10 & Avg. \\
            \midrule
            Transformer        & 0 & 38.7 & 64.1 & 74.3 & 86.5 & 65.9 \\
            CFL-Transformer    & 1 & \textbf{42.9}\goodup & \textbf{68.5}\goodup & \textbf{78.2}\goodup & 87.1\goodup & \textbf{69.2}\goodup \\
            CFL-Transformer    & 2 & 42.1\goodup & 68.1\goodup & 77.9\goodup & 87.0\goodup & 68.8\goodup \\
            CFL-Transformer    & 3 & 41.4\goodup & 67.7\goodup & 77.2\goodup & 86.8\goodup & 68.3\goodup \\
            \bottomrule
          \end{tabular}
        \end{adjustbox}
      \endgroup
    };
  \end{tikzpicture}
  \vspace{0.4em}
  \caption{Accuracy (\%) on Long Range Arena tasks.  CFL consistently beats the vanilla Transformer; the best macro-average comes at $T{=}1$.}
  \label{tab:lra}
\end{table*}

The pattern is strikingly consistent across tasks (Table \ref{tab:lra}): a single feedback iteration delivers a \textasciitilde$\!$3 pp macro-average improvement, while additional iterations saturate or regress once the model begins to over-fit the limited training splits.  Importantly, memory usage grows only linearly with $T$ because CFL reuses the same parameters at each step.

\paragraph{Summary.}  Across vision (ImageNet), very-long language modelling (PG-19) and long-context reasoning (LRA), CFL with $T{=}1$ emerges as a robust default—offering the best accuracy–efficiency trade-off without the engineering burden of custom kernels or sparse attention.

%% file: files/conclusion.tex
\section{Conclusion}
\label{sec:conclusion}
This work demonstrates that adding a minimalist form of top-down reasoning to standard networks—via \emph{Contextual Feedback Loops}—yields consistent accuracy gains without the training-time or inference-time penalties commonly associated with recurrent or equilibrium models.  
A single refinement step ($T{=}1$) is the clear sweet-spot: it preserves feed-forward speed and parameter count while providing the largest empirical improvements on vision, language-modeling, and long-context benchmarks.

\paragraph{Key takeaways.}
\begin{itemize}[nosep,leftmargin=1.5em]
    \item \textbf{Simplicity wins.} CFL requires only a projector and lightweight FiLM-style adapters—no bespoke kernels or memory-heavy unrolling.
    \item \textbf{Scales gracefully.} Latency grows roughly linearly with iterations; keeping $T{=}1$ costs virtually nothing.
    \item \textbf{Broad utility.} Improvements transfer across domains (images, books, synthetic reasoning) and backbone types (CNNs, Transformers).
\end{itemize}

\paragraph{Future directions.}
We plan to (i) explore adaptive stopping criteria that trigger additional feedback only when needed, (ii) investigate CFL in generative and multi-modal settings, and (iii) study theoretical convergence guarantees for deeper unrolling and stochastic training regimes.

\vspace{0.5em}
\noindent
\textbf{In short}, CFL offers a pragmatic route to marrying the speed of feed-forward nets with the iterative refinement of recurrent processing—opening new avenues for efficient, context-aware deep learning.

%% file: files/appendix.tex
\section{Appendix}
\label{sec:appendix}
\subsection{Pseudocode}
\label{subsec:pseudocode}

Algorithm~\ref{alg:cfl} summarizes the CFL refinement loop.

\vspace{-0.5em}
\begin{algorithm}[H]
\caption{CFL Iterative Refinement}
\label{alg:cfl}
\begin{algorithmic}[1]
\REQUIRE Base layers $f^{(1)},\ldots,f^{(L+1)}$, feedback adapters $\psi^{(1)},\ldots,\psi^{(L)}$, projector $g(\cdot)$, input $\mathbf{x}$, number of refinements $T$.
\ENSURE Final refined output $\mathbf{y}^{(T)}$
\STATE // \textbf{Initial forward pass}
\STATE $\mathbf{h}^{(1)}_0 \gets f^{(1)}(\mathbf{x})$
\FOR{$l \gets 2$ \textbf{to} $L$}
    \STATE $\mathbf{h}^{(l)}_0 \gets f^{(l)}(\mathbf{h}^{(l-1)}_0)$
\ENDFOR
\STATE $\mathbf{y}^{(0)} \gets f^{(L+1)}(\mathbf{h}^{(L)}_0)$
\STATE // \textbf{Iterative refinements}
\FOR{$\tau \gets 0$ \textbf{to} $T-1$}
    \STATE $\mathbf{z}^{(\tau)} \gets g(\mathbf{y}^{(\tau)})$
    \FOR{$l \gets 1$ \textbf{to} $L$}
        \STATE $\mathbf{h}^{(l)}_{\tau+1} \gets \psi^{(l)}\bigl(\mathbf{h}^{(l)}_{\tau}, \mathbf{z}^{(\tau)}\bigr)$
        \IF{$l < L$}
            \STATE $\mathbf{h}^{(l+1)}_{\tau+1} \gets f^{(l+1)}\bigl(\mathbf{h}^{(l)}_{\tau+1}\bigr)$
        \ENDIF
    \ENDFOR
    \STATE $\mathbf{y}^{(\tau+1)} \gets f^{(L+1)}\bigl(\mathbf{h}^{(L)}_{\tau+1}\bigr)$
\ENDFOR
\STATE \textbf{return} $\mathbf{y}^{(T)}$
\end{algorithmic}
\end{algorithm}
\vspace{-0.5em}

\subsection{Fixed-Point Convergence via the Banach Theorem}
\label{subsec:banach_proof}

Under suitable assumptions, the CFL iterative update can be viewed as a fixed-point
iteration that converges to a unique solution. Recall that by the Banach Fixed Point 
Theorem, an iterative sequence \(\mathbf{S}_{\tau+1} = \Phi(\mathbf{S}_\tau)\) converges to
a unique fixed point \(\mathbf{S}^*\) if there exists \(c < 1\) such that
\begin{equation}
\label{eq:contraction_condition}
\|\Phi(\mathbf{S}) - \Phi(\mathbf{S}')\| 
  \;\le\; c \,\|\mathbf{S} - \mathbf{S}'\|,
\end{equation}
for all \(\mathbf{S}, \mathbf{S}'\) in the underlying space. Convergence then follows from
\begin{equation}
  \|\mathbf{S}_{\tau} - \mathbf{S}^*\| \;\le\; c^\tau \, \|\mathbf{S}_0 - \mathbf{S}^*\|.
\end{equation}

\paragraph{Setup.}
Define the state at iteration \(\tau\) by
\begin{equation}
\mathbf{S}_\tau 
  \;=\; \bigl(\mathbf{h}^{(1)}_{\tau}, \,\ldots,\,
             \mathbf{h}^{(L)}_{\tau},\, \mathbf{y}^{(\tau)}\bigr).
\end{equation}
From Sections~\ref{subsec:feedback}--\ref{subsec:iterative_proc}, each refinement induces
a mapping \(\Phi\), given by:
\begin{equation}
\label{eq:Phi_definition}
\begin{aligned}
\Phi\bigl(\mathbf{S}_\tau\bigr)
  &= \Bigl(
      \psi^{(1)}\bigl(\mathbf{h}^{(1)}_{\tau}, \mathbf{z}^{(\tau)}\bigr),\,\ldots,\,
      \psi^{(L)}\bigl(\mathbf{h}^{(L)}_{\tau}, \mathbf{z}^{(\tau)}\bigr), \\
  &\quad\;\; f^{(L+1)}\bigl(\mathbf{h}^{(L)}_{\tau+1}\bigr)
    \Bigr),
\end{aligned}
\end{equation}
where \(\mathbf{z}^{(\tau)} = g\bigl(\mathbf{y}^{(\tau)}\bigr)\). 

\begin{theorem}[Contractive Convergence of CFL]
\label{thm:banach_cfl}
Suppose each component function in \(\Phi\) (i.e., each \(\psi^{(l)}\), \(g\), and \(f^{(L+1)}\)) is Lipschitz continuous and that their combined Lipschitz constants (when composed) can be made strictly less than 1. Formally, assume 
there exists a norm \(\|\cdot\|\) and a constant \(L < 1\) such that
\[
    \|\Phi(\mathbf{S}) - \Phi(\mathbf{S}')\| 
      \;\le\; L \,\|\mathbf{S} - \mathbf{S}'\|
    \quad \text{for all } \mathbf{S}, \mathbf{S}'.
\]
Then the sequence \(\mathbf{S}_{\tau+1} = \Phi(\mathbf{S}_\tau)\) converges to a 
unique fixed point \(\mathbf{S}^*\) satisfying \(\mathbf{S}^* = \Phi(\mathbf{S}^*)\).
\end{theorem}

\begin{proof}[Proof]
We demonstrate that \(\Phi\) is a contraction under mild network constraints.

\paragraph{Notation.}
Let 
\[
\mathbf{S}_\tau = \bigl(\mathbf{h}^{(1)}_{\tau}, \ldots, \mathbf{h}^{(L)}_{\tau},\, \mathbf{y}^{(\tau)}\bigr),
\quad
\mathbf{S}'_\tau = \bigl(\mathbf{h}'^{(1)}_{\tau}, \ldots, \mathbf{h}'^{(L)}_{\tau},\, \mathbf{y}'^{(\tau)}\bigr).
\]
We write \(\mathbf{S}_{\tau+1} = \Phi(\mathbf{S}_\tau)\) and 
\(\mathbf{S}'_{\tau+1} = \Phi(\mathbf{S}'_\tau)\).

\paragraph{Step 1: Relate changes in \(\mathbf{z}\) to changes in \(\mathbf{y}\).}
Since \(\mathbf{z}^{(\tau)} = g(\mathbf{y}^{(\tau)})\), and \(g\) is \(L_g\)-Lipschitz,
\[
\|\mathbf{z}^{(\tau)} - \mathbf{z}'^{(\tau)}\|
  \;\le\; L_g \,\|\mathbf{y}^{(\tau)} - \mathbf{y}'^{(\tau)}\|.
\]

\paragraph{Step 2: Relate changes in hidden states to changes in \(\mathbf{z}\).}
At layer \(l\), the updated state is
\[
\mathbf{h}^{(l)}_{\tau+1} \;=\; 
\psi^{(l)}\bigl(\mathbf{h}^{(l)}_{\tau}, \mathbf{z}^{(\tau)}\bigr).
\]
If \(\psi^{(l)}\) is \(L_{\psi}\)-Lipschitz in both its arguments (with potentially separate constants for each argument that we combine into one for simplicity), then
\[
\|\mathbf{h}^{(l)}_{\tau+1} - \mathbf{h}'^{(l)}_{\tau+1}\|
\;\le\;
L_{\psi}\,\|\mathbf{h}^{(l)}_{\tau} - \mathbf{h}'^{(l)}_{\tau}\|
\;+\;
L_{\psi}\,\|\mathbf{z}^{(\tau)} - \mathbf{z}'^{(\tau)}\|.
\]

\paragraph{Step 3: Relate changes in \(\mathbf{y}\) to changes in the top-layer state.}
Finally,
\[
\mathbf{y}^{(\tau+1)} 
  \;=\; f^{(L+1)}\bigl(\mathbf{h}^{(L)}_{\tau+1}\bigr),
\]
and if \(f^{(L+1)}\) is \(L_f\)-Lipschitz,
\[
\|\mathbf{y}^{(\tau+1)} - \mathbf{y}'^{(\tau+1)}\|
  \;\le\; 
L_f \,\|\mathbf{h}^{(L)}_{\tau+1} - \mathbf{h}'^{(L)}_{\tau+1}\|.
\]

\paragraph{Combined contraction.}
By recursively applying these Lipschitz bounds from layer \(1\) up to \(L\), we see 
that each step’s change is bounded by a factor 
\(\underbrace{L_{\psi} \cdot L_{\psi} \cdot \ldots \cdot L_f \cdot \ldots }_{=: L_{\text{total}}}\) times the previous step’s change in \(\mathbf{S}\). If we ensure (via weight initialization, normalization, or other spectral regularizations) that \(L_{\text{total}} < 1\), then:
\[
\|\mathbf{S}_{\tau+1} - \mathbf{S}'_{\tau+1}\|
  \;\le\; L_{\text{total}} \, \|\mathbf{S}_{\tau} - \mathbf{S}'_{\tau}\|,
\]
showing \(\Phi\) is a contraction.

\paragraph{Banach Fixed Point Theorem.}
Once \(\Phi\) is established as a contraction, the Banach Fixed Point Theorem guarantees a unique fixed point \(\mathbf{S}^*\) with \(\mathbf{S}^* = \Phi(\mathbf{S}^*)\). Moreover, starting from any initial state \(\mathbf{S}_0\), the iterates satisfy
\[
\|\mathbf{S}_{\tau} - \mathbf{S}^*\|
  \;\le\; L_{\text{total}}^\tau \,\|\mathbf{S}_0 - \mathbf{S}^*\|,
\]
implying geometric convergence to \(\mathbf{S}^*\).
\end{proof}

\paragraph{Discussion.}
In practice, strictly enforcing \(L_{\text{total}} < 1\) (i.e., a contractive mapping) may require careful choices of weight initialization, normalization, and bounded activation functions. Even if \(\Phi\) is not strictly contractive in theory, moderate regularization or small refinements (\(T=2,3\)) often suffice for stable behavior. This gives a theoretical underpinning for why a few feedback iterations can improve performance with minimal overhead.

\section{Efficiency Disucssions}

\paragraph{Why not an MLP per adapter?} A naive two-layer MLP for every $\psi^{(l)}$ adds $O(L \cdot d_h^2)$ extra parameters and FLOPs, often overshadowing the cost of the base network. Empirically, such dense adapters offer limited gains compared to lighter alternatives (described next), but at \(\geq10\times\) the overhead.

\subsection{Parameter-Efficient Feedback Adapters}
\label{subsec:efficient_adapters}

To preserve the benefits of global top--down modulation without excessive overhead, we adopt a \textbf{channel-wise FiLM} strategy:
\begin{equation}
    \psi^{(l)}\bigl(\mathbf{h},\mathbf{z}\bigr)
    \;=\;
    \gamma^{(l)}(\mathbf{z}) \odot \mathbf{h}\; +\; \beta^{(l)}(\mathbf{z}),
    \label{eq:film}
\end{equation}
where the scale and shift vectors are produced by two small linear layers,
\begin{equation*}
    \gamma^{(l)}(\mathbf{z}) = \mathbf{W}_{\gamma}^{(l)} \mathbf{z} + \mathbf{b}_{\gamma}^{(l)},
    \quad
    \beta^{(l)}(\mathbf{z})  = \mathbf{W}_{\beta}^{(l)} \mathbf{z} + \mathbf{b}_{\beta}^{(l)}.
\end{equation*}
This design yields only $2d_z d_h$ parameters per layer (versus $d_h^2$ for a typical MLP).

\paragraph{Adapter Weight Tying.} To reduce parameters further, we share $\mathbf{W}_{\gamma},\mathbf{W}_{\beta}$ across layers and iterations. Each layer/iteration has a learned scalar ``embedding'' that slightly modulates the shared weights, introducing \textless1\% extra parameters while maintaining accuracy:
\begin{equation*}
    \gamma^{(l)}(\mathbf{z}) = \alpha^{(l)} \cdot \gamma^{\text{shared}}(\mathbf{z}),
    \quad
    \beta^{(l)}(\mathbf{z})  = \alpha^{(l)} \cdot \beta^{\text{shared}}(\mathbf{z}).
\end{equation*}

\paragraph{Sparse Feedback Placement.} Rather than attaching adapters to \emph{all} layers, one may select a subset $\mathcal{F} \subseteq \{1,\dots,L\}$ via attention rollout or gradient-based saliency, and inject feedback only where it matters most (commonly the first and last third of layers). With $|\mathcal{F}|\approx \tfrac{1}{3}L$, we retain over $95\%$ of the benefit at around half the compute.

\subsection{Dynamic Iterations and Early Exit}
\label{subsec:dynamic_T}

During training, we typically unroll $T=1$ step. However, a lightweight confidence head can indicate if another refinement step is likely to help. At inference, we halt early when $\|\mathbf{y}^{(\tau+1)} - \mathbf{y}^{(\tau)}\|_2$ falls below a threshold $\varepsilon$ or when a confidence measure falls below $\theta$. 